\documentclass[a4paper]{article} 

\usepackage{a4wide}
\usepackage{inputenc}
\usepackage{amssymb, amsmath}
\usepackage{amsfonts}
\usepackage{amsthm}
\usepackage{graphicx}

\DeclareMathOperator*{\argmin}{argmin}

\DeclareMathOperator*{\tr}{tr}
\DeclareMathOperator*{\diag}{diag}
\DeclareMathOperator*{\Diag}{Diag}

\newtheorem{definition}{Definition}
\newtheorem{proposition}{Proposition}

\DeclareMathAlphabet{\mathonebb}{U}{bbold}{m}{n}
\newcommand{\ind}{\ensuremath{\mathonebb{1}}}

\title{\textbf{Trace Lasso: a trace norm regularization \\
for correlated designs}}

\author{
\hspace{\textwidth} \\
\textbf{Edouard Grave} \hfill \texttt{edouard.grave@inria.fr} \\[.2cm]
\textbf{Guillaume Obozinski} \hfill \texttt{guillaume.obozinski@inria.fr}  \\[.2cm]
\textbf{Francis Bach} \hfill \texttt{francis.bach@inria.fr}  \\[.2cm]
\begin{flushleft}
\textit{INRIA - Sierra Project-Team} \\
\textit{Laboratoire d'Informatique de l'\'Ecole Normale Sup\'erieure} \\
\textit{Paris, France}
\end{flushleft}
}

\begin{document}

\begin{center}
\LARGE
\textbf{Trace Lasso: a trace norm regularization \\
for correlated designs}
\end{center}

\vspace{10mm}

{
\large \noindent
\textbf{Edouard Grave} \hfill \texttt{edouard.grave@inria.fr} \\
\textbf{Guillaume Obozinski} \hfill \texttt{guillaume.obozinski@inria.fr} \\
\textbf{Francis Bach} \hfill \texttt{francis.bach@inria.fr} \\
}
\textit{INRIA - Sierra project-team} \\
\textit{Laboratoire d'Informatique de l'\'Ecole Normale Sup\'erieure} \\
\textit{Paris, France}

\vspace{15mm}

\begin{abstract}
Using the $\ell_1$-norm to regularize the estimation of  the parameter vector of a linear model leads to an unstable estimator when covariates are highly correlated. In this paper, we introduce a new penalty function which takes into account the correlation of the design matrix to stabilize the estimation. This norm, called the trace Lasso, uses the trace norm, which is a convex surrogate of the rank, of the selected covariates as the criterion of model complexity. We analyze the properties of our norm, describe an optimization algorithm based on reweighted least-squares, and illustrate the behavior of this norm on synthetic data, showing that it is more adapted to strong correlations than competing methods such as the elastic net.
\end{abstract}

\section{Introduction}
The concept of parsimony is central in many scientific domains. In the context of statistics, signal processing or machine learning, it takes the form of variable or feature selection problems, and is commonly used in two situations: first, to make the model or the prediction more interpretable or cheaper to use, i.e., even if the underlying problem does not admit sparse solutions, one looks for the best sparse approximation. Second, sparsity can also be used given prior knowledge that the model should be sparse. Many methods have been designed to learn sparse models, namely methods based on combinatorial optimization \cite{mallat1993, zhang2008}, Bayesian inference~\cite{seeger2008} or convex optimization~\cite{tibshirani1996, chen1999}.

In this paper, we focus on the regularization by sparsity-inducing norms. The simplest example of such norms is the $\ell_1$-norm, leading to the Lasso, when used within a least-squares framework. In recent years, a large body of work has shown that the Lasso was performing optimally in high-dimensional low-correlation settings, both in terms of prediction~\cite{bickel2009}, estimation of parameters or estimation of supports~\cite{zhao2006, wainwright2009}. However, most data exhibit strong correlations, with various correlation structures, such as clusters (i.e., close to block-diagonal covariance matrices)  or sparse graphs, such as for example  problems involving sequences (in which case, the covariance matrix is close to a Toeplitz matrix~\cite{golub1996}). In these situations, the Lasso is known to have stability problems: although its predictive performance is not disastrous, the selected predictor may vary a lot (typically, given two correlated variables, the Lasso will only select one of the two, at random).

Several remedies have been proposed to this instability. First, the elastic net~\cite{zou2005} adds a strongly convex penalty term (the squared $\ell_2$-norm) that will stabilize selection (typically, given two correlated variables, the elastic net will  select the two variables).  However, it is blind to the exact correlation structure, and while strong convexity is required for some variables, it is not for other variables. Another solution is to consider the group Lasso, which will divide the predictors into groups and penalize the sum of the $\ell_2$-norm of these groups~\cite{yuan2006}. This is known to accomodate strong correlations within groups~\cite{bach2008}; however it requires to know the group in advance, which is not always possible. A third line of research has focused on sampling-based techniques~\cite{bach2008bolasso, liu2010, meinshausen2010}.

An ideal regularizer should thus be adapted to the design (like the group Lasso), but without requiring human intervention (like the elastic net); it should thus add strong convexity only where needed, and not modifying variables where things behave correctly. In this paper, we propose a new norm towards this end.

More precisely we make the following contributions:

\begin{itemize}
\item We propose in Section~\ref{def}  a new norm based on the trace norm (a.k.a.~nuclear norm) that interpolates between the $\ell_1$-norm and the $\ell_2$-norm depending on correlations.
\item We show that there is a unique minimum when penalizing with this norm in Section~\ref{singleopt}.
\item We provide optimization algorithms based on reweighted least-squares in Section~\ref{opt}.
\item We study the second-order expansion around independence and relate to existing work on including correlations in Section~\ref{dev_lasso}.
\item We perform synthetic experiments in Section~\ref{exp}, where we show that the trace Lasso outperforms existing norms in strong-correlation regimes.
\end{itemize}

\paragraph{Notations.}
Let $\mathbf{M} \in \mathbb{R}^{n \times p}$. The columns of $\mathbf{M}$ are noted using superscript, i.e.,  $\mathbf{M}^{(i)}$ denotes the $i$-th column,  while the rows are noted using subscript, i.e.,  $\mathbf{M}_i$ denotes the $i$-th row. For $\mathbf{M} \in \mathbb{R}^{p \times p}$, $\diag(\mathbf{M}) \in \mathbb{R}^{p}$ is the diagonal of the matrix $\mathbf{M}$, while for $\mathbf{u} \in \mathbb{R}^p$, $\Diag(\mathbf{u}) \in \mathbb{R}^{p \times p}$ is the diagonal matrix whose diagonal elements are the $u_i$. Let $S$ be a subset of $\{1, ..., p \}$, then $\mathbf{u}_S$ is the vector $\mathbf{u}$ restricted to the support $S$, with $0$ outside the support $S$. We denote by $\mathbb{S}_p$ the set of symmetric matrices of size $p$. We will use various matrix norms, here are the notations we use:
\begin{itemize}
\item $\| \mathbf{M} \|_*$ is the trace norm, \emph{i.e.}, the sum of the singular values of the matrix $\mathbf{M}$, 
\item $\| \mathbf{M} \|_{op}$ is the operator norm, \emph{i.e.}, the maximum singular value of the matrix $\mathbf{M}$,
\item $\| \mathbf{M} \|_{F}$ is the Frobenius norm, \emph{i.e.}, the $\ell_2$-norm of the singular values, which is also equal to $\sqrt{\tr (\mathbf{M}^{\top} \mathbf{M})}$,
\item $\| \mathbf{M} \|_{2,1}$ is the sum of the $\ell_2$-norm of the columns of $\mathbf{M}$:
$
\displaystyle{\| \mathbf{M} \|_{2,1} = \sum_{i = 1}^p \| \mathbf{M}^{(i)} \|_2}
$.
\end{itemize}

\section{Definition and properties of the trace Lasso}
\label{def}
We consider the problem of predicting $y \in \mathbb{R}$, given a vector $\mathbf{x} \in \mathbb{R}^p$, assuming a linear model
\begin{equation*}
y = \mathbf{w}^{\top}\mathbf{x} + \varepsilon,
\end{equation*}
where $\varepsilon$ is (Gaussian) noise with mean $0$ and variance $\sigma^2$. Given a training set $\mathbf{X} = (\mathbf{x}_1, ..., \mathbf{x}_n)^{\top}~\in~\mathbb{R}^{n \times p}$ and $\mathbf{y} = (y_1, ..., y_n)^{\top} \in \mathbb{R}^{n}$, a widely used method to estimate the parameter vector $\mathbf{w}$ is the penalized empirical risk minimization
\begin{equation}
\label{erm}
\mathbf{\hat w} \in \argmin_{\mathbf{w}} \ \frac{1}{n} \sum_{i=1}^n \ell(y_i, \mathbf{w}^{\top} \mathbf{x}_i) + \lambda f (\mathbf{w}),
\end{equation}
where $\ell$ is a loss function used to measure the error we make by predicting $\mathbf{w}^{\top} \mathbf{x}_i$ instead of $y_i$, while $f$ is a regularization term used to penalize complex models. This second term helps avoiding overfitting, especially in the case where we have many more parameters than observation, \emph{i.e.},~$n~\ll~p$. 

\subsection{Related work}
We will now present some classical penalty functions for linear models which are widely used in the machine  learning  and statistics community. The first one, known as Tikhonov regularization~\cite{tikhonov1963} or ridge regression~\cite{hoerl1970}, is the squared $\ell_2$-norm. When used with the square loss, estimating the parameter vector $\mathbf{w}$ is done by solving a linear system. One of the main drawbacks of this penalty function is the fact that it does not perform variable selection and thus does not behave well in sparse high-dimensional settings.

Hence, it is natural to penalize linear models by the number of variables used by the model. Unfortunately, this criterion, sometimes denoted by $\| \cdot \|_0$ ($\ell_0$-penalty), is not convex and solving the problem in Eq.~(\ref{erm}) is generally NP-hard~\cite{davis1997}. Thus, a convex relaxation for this problem was introduced, replacing the size of the selected subset by the $\ell_1$-norm of $\mathbf{w}$. This estimator is known as the Lasso~\cite{tibshirani1996} in the statistics community and basis pursuit~\cite{chen1999} in signal processing. It was later shown that under some assumptions, the two problems were in fact equivalent (see for example~\cite{candes2005} and references therein). 

When two predictors are highly correlated, the Lasso has a very unstable behavior: it may only select the variable that is the most correlated with the residual. On the other hand, the Tikhonov regularization tends to shrink coefficients of correlated variables together, leading to a very stable behavior. In order to get the best of both worlds, stability and variable selection, Zou and Hastie introduced the elastic net~\cite{zou2005}, which is the sum of the $\ell_1$-norm and squared $\ell_2$-norm. Unfortunately, this estimator needs two regularization parameters and is not adaptive to the precise correlation structure of the data. Some authors also proposed to use pairwise correlations between predictors to interpolate more adaptively between the $\ell_1$-norm and squared $\ell_2$-norm, by introducing the pairwise elastic net~\cite{lorbert2010} (see comparisons with our approach in Section~\ref{exp}).

Finally, when one has more knowledge about the data, for example clusters of variables that should be selected together, one can use the group Lasso~\cite{yuan2006}. Given a partition $(S_i)$ of the set of variables, it is defined as the sum of the $\ell_2$-norms of the restricted vectors $\mathbf{w}_{S_i}$:
\begin{equation*}
\| \mathbf{w} \|_{GL} = \sum_{i=1}^k \| \mathbf{w}_{S_i} \|_2.
\end{equation*}
The effect of this penalty function is to introduce sparsity at the group level: variables in a group are selected altogether. One of the main drawback of this method, which is also sometimes one of its quality, is the fact that one needs to know the partition of the variables, and so one needs to have a good knowledge of the data.

\subsection{The ridge, the Lasso and the trace Lasso}
In this section, we show that Tikhonov regularization and the Lasso penalty can be viewed as norms of the matrix $\mathbf{X} \Diag (\mathbf{w})$. We then introduce a new norm involving this matrix. 

The solution of empirical risk minimization penalized by the $\ell_1$-norm or $\ell_2$-norm is not equivariant by rescaling of the predictors $\mathbf{X}^{(i)}$, so it is common to normalize the predictors. When normalizing the predictors $\mathbf{X}^{(i)}$, and penalizing by Tikhonov regularization or by the Lasso, people are implicitly using a regularization term that depends on the data or design matrix $\mathbf{X}$. In fact, there is an equivalence between normalizing the predictors and not normalizing them, using the two following reweighted $\ell_2$ and $\ell_1$-norms instead of the Tikhonov regularization and the Lasso:
\begin{equation}
\label{lth}
\| \mathbf{w} \|_{2}^2 = \sum_{i=1}^p \| \mathbf{X}^{(i)} \|_2^2 \ w_i^2
\hspace{10mm} \text{and} \hspace{10mm}
\| \mathbf{w} \|_{1} = \sum_{i=1}^p \| \mathbf{X}^{(i)} \|_2 \ | w_i |.
\end{equation}
These two norms can be expressed using the matrix $\mathbf{X} \Diag(\mathbf{w})$:
\begin{equation*}
\| \mathbf{w} \|_{2} = \| \mathbf{X} \Diag (\mathbf{w}) \|_{F}
\hspace{10mm} \text{and} \hspace{10mm}
\| \mathbf{w} \|_{1} = \| \mathbf{X} \Diag (\mathbf{w}) \|_{2,1},
\end{equation*}
and a natural question arises: are there other relevant choices of functions or matrix norms? A classical measure of the complexity of a model is the number of predictors used by this model, which is equal to the size of the support of $\mathbf{w}$. This penalty being non-convex, people use its convex relaxation, which is the $\ell_1$-norm, leading to the Lasso.

Here, we propose a different measure of complexity which can be shown to be more adapted in model selection settings~\cite{hastie2001}: the dimension of the subspace spanned by the selected predictors. This is equal to the rank of the selected predictors, or also to the rank of the matrix $\mathbf{X} \Diag(\mathbf{w})$. As for the size of the support, this function is non-convex, and we propose to replace it by a convex surrogate, the \emph{trace norm}, leading to the following penalty that we call ``trace Lasso'':
\begin{equation*}
\Omega(\mathbf{w}) = \| \mathbf{X} \Diag(\mathbf{w}) \|_*.
\end{equation*}
The trace Lasso has some interesting properties: if all the predictors are orthogonal, then, it is equal to the $\ell_1$-norm. Indeed, we have the decomposition:
\begin{equation*}
\mathbf{X} \Diag(\mathbf{w}) = \sum_{i=1}^p \left( \| \mathbf{X}^{(i)} \|_2 w_i \right) \frac{\mathbf{X}^{(i)}}{\| \mathbf{X}^{(i)}\|_2} \mathbf{e}_i^{\top},
\end{equation*}
where $\mathbf{e}_i$ are the vectors of the canonical basis. Since the predictors are orthogonal and the $\mathbf{e}_i$ are orthogonal too, this gives the singular value decomposition of $\mathbf{X} \Diag(\mathbf{w})$ and we get
\begin{equation*}
\| \mathbf{X} \Diag(\mathbf{w}) \|_* = \sum_{i=1}^p \| \mathbf{X}^{(i)} \|_2 | w_i |  = \| \mathbf{X} \Diag(\mathbf{w}) \|_{2,1}.
\end{equation*}
On the other hand, if all the predictors are equal to $\mathbf{X}^{(1)}$, then
\begin{equation*}
\mathbf{X} \Diag(\mathbf{w}) = \mathbf{X}^{(1)} \mathbf{w}^{\top},
\end{equation*}
and we get $\| \mathbf{X} \Diag(\mathbf{w}) \|_* = \| \mathbf{X}^{(1)} \|_2 \| \mathbf{w}\|_2 = \|\mathbf{X} \Diag (\mathbf{w}) \|_F$, which is equivalent to the Tikhonov regularization. Thus when two predictors are strongly correlated, our norm will behave like the Tikhonov regularization, while for almost uncorrelated predictors, it will behave like the Lasso. 

Always having a unique minimum is an important property for a statistical estimator, as it is a first step towards stability. The trace Lasso, by adding strong convexity exactly in the direction of highly correlated covariates, always has a unique minimum, and is much more stable than the Lasso. 
\begin{proposition}
\label{unicity}
If the loss function $\ell$ is strongly convex with respect to its second argument, then the solution of the empirical risk minimization penalized by the trace Lasso, i.e., Eq.~(\ref{erm}), is unique.
\end{proposition}
\label{singleopt}
The technical proof of this proposition is given in appendix \ref{proof_prop_1}, and consists of showing that in the flat directions of the loss function, the trace Lasso is strongly convex. 

\subsection{A new family of penalty functions}

In this section, we introduce a new family of penalties, inspired by the trace Lasso, allowing us to write the $\ell_1$-norm, the $\ell_2$-norm and the newly introduced trace Lasso as special cases. In fact, we note that $\| \Diag(\mathbf{w}) \|_* = \| \mathbf{w} \|_1$ and $\| p^{-1/2} \mathbf{1}^{\top} \Diag(\mathbf{w}) \|_* = \| \mathbf{w}^{\top} \|_* = \| \mathbf{w} \|_2$. In other words, we can express the $\ell_1$ and $\ell_2$-norms of $\mathbf{w}$ using the trace norm of a given matrix times the matrix $\Diag (\mathbf{w})$. A natural question to ask is: what happens when using a matrix $\mathbf{P}$ other than the identity or the line vector $p^{-1/2}\mathbf{1}^{\top}$, and what are good choices of such matrices? Therefore, we introduce the following family of penalty functions:
\begin{definition}
Let $\mathbf{P} \in \mathbb{R}^{k \times p}$, all of its columns having unit norm. We introduce the norm $\Omega_{\mathbf{P}}$ as
\begin{equation*}
\Omega_{\mathbf{P}} (\mathbf{w}) = \| \mathbf{P} \Diag(\mathbf{w}) \|_*.
\end{equation*}
\end{definition}
\begin{proof}
The positive homogeneity and triangle inequality are direct consequences of the linearity of $\mathbf{w} \mapsto \mathbf{P} \Diag(\mathbf{w})$ and the fact that $\| \cdot \|_*$ is a norm. Since all the columns of $\mathbf{P}$ are not equal to zero, we have 
\begin{equation*}
\mathbf{P} \Diag(\mathbf{w}) = 0 \Leftrightarrow \mathbf{w} = 0,
\end{equation*}
and so, $\Omega_{\mathbf{P}}$ separates points and is a norm.
\end{proof}
As stated before, the $\ell_1$ and $\ell_2$-norms are special cases of the family of norms we just introduced. Another important penalty that can be expressed as a special case is the group Lasso, with non-overlapping groups. Given a partition $(S_j)$ of the set $\{1, ..., p \}$, the group Lasso is defined by
\begin{equation*}
\| \mathbf{w} \|_{GL}  = \sum_{S_j} \| \mathbf{w}_{S_j} \|_2.
\end{equation*}
We define the matrix $\mathbf{P}^{GL}$ by
\begin{equation*}
\mathbf{P}^{GL}_{ij} = \left\{
\begin{array}{ll}
1 / \sqrt{|S_k|}& \text{ if $i$ and $j$ are in the same group $S_k$,} \\
0 & \text{ otherwise.}
\end{array}
\right.
\end{equation*}
Then, 
\begin{equation}
\label{svd_gl}
\mathbf{P}^{GL} \Diag (\mathbf{w}) = \sum_{S_j} \frac{\mathbf{1}_{S_j}}{\sqrt{|S_j|}} \mathbf{w}_{S_j}^{\top} .
\end{equation}
Using the fact that $(S_j)$ is a partition of $\{1,...,p \}$, the vectors $\mathbf{1}_{S_j}$ are orthogonal and so are the vectors $\mathbf{w}_{S_j}$. Hence, after normalizing the vectors, Eq.~(\ref{svd_gl}) gives a singular value decomposition of $\mathbf{P}^{GL} \Diag(\mathbf{w})$ and so the group Lasso penalty can be expressed as a special case of our family of norms:
\begin{equation*}
\| \mathbf{P}^{GL} \Diag(\mathbf{w}) \|_* = \sum_{S_j} \| \mathbf{w}_{S_j}\|_2 = \| \mathbf{w} \|_{GL}.
\end{equation*}

In the following proposition, we show that our norm only depends on the value of $\mathbf{P}^{\top} \mathbf{P}$. This is an important property for the trace Lasso, where $\mathbf{P} = \mathbf{X}$, since it underlies the fact that this penalty only depends on the correlation matrix $\mathbf{X}^{\top} \mathbf{X}$ of the covariates.
\begin{proposition}
Let $\mathbf{P} \in \mathbb{R}^{k \times p}$, all of its columns having unit norm. We have
\begin{equation*}
\Omega_{\mathbf{P}}(\mathbf{w}) = \| (\mathbf{P}^{\top} \mathbf{P})^{1/2} \Diag(\mathbf{w}) \|_*.
\end{equation*}
\end{proposition}

\begin{figure}[t]
\centering
\includegraphics[width=0.32 \textwidth]{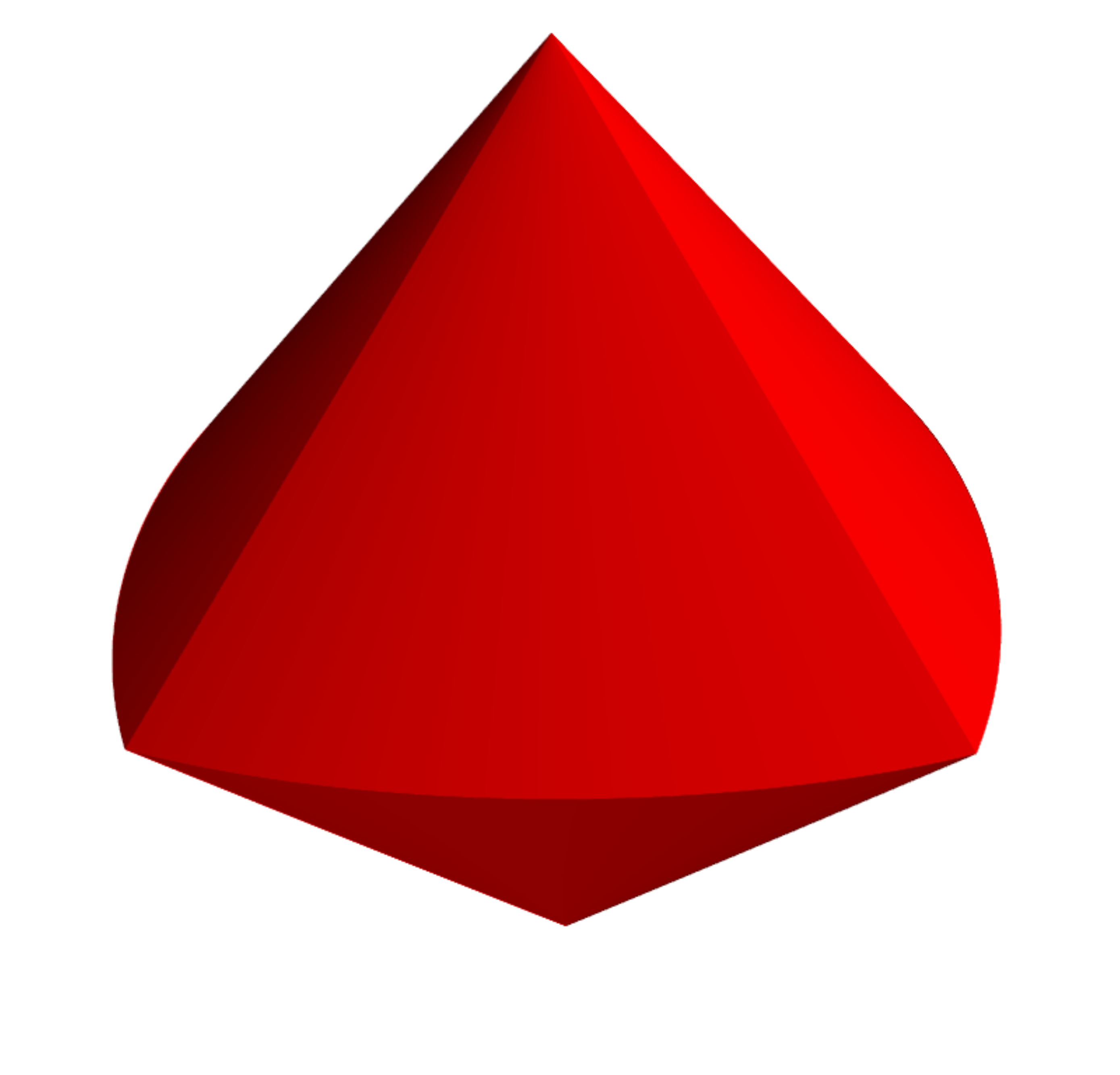}
\includegraphics[width=0.32 \textwidth]{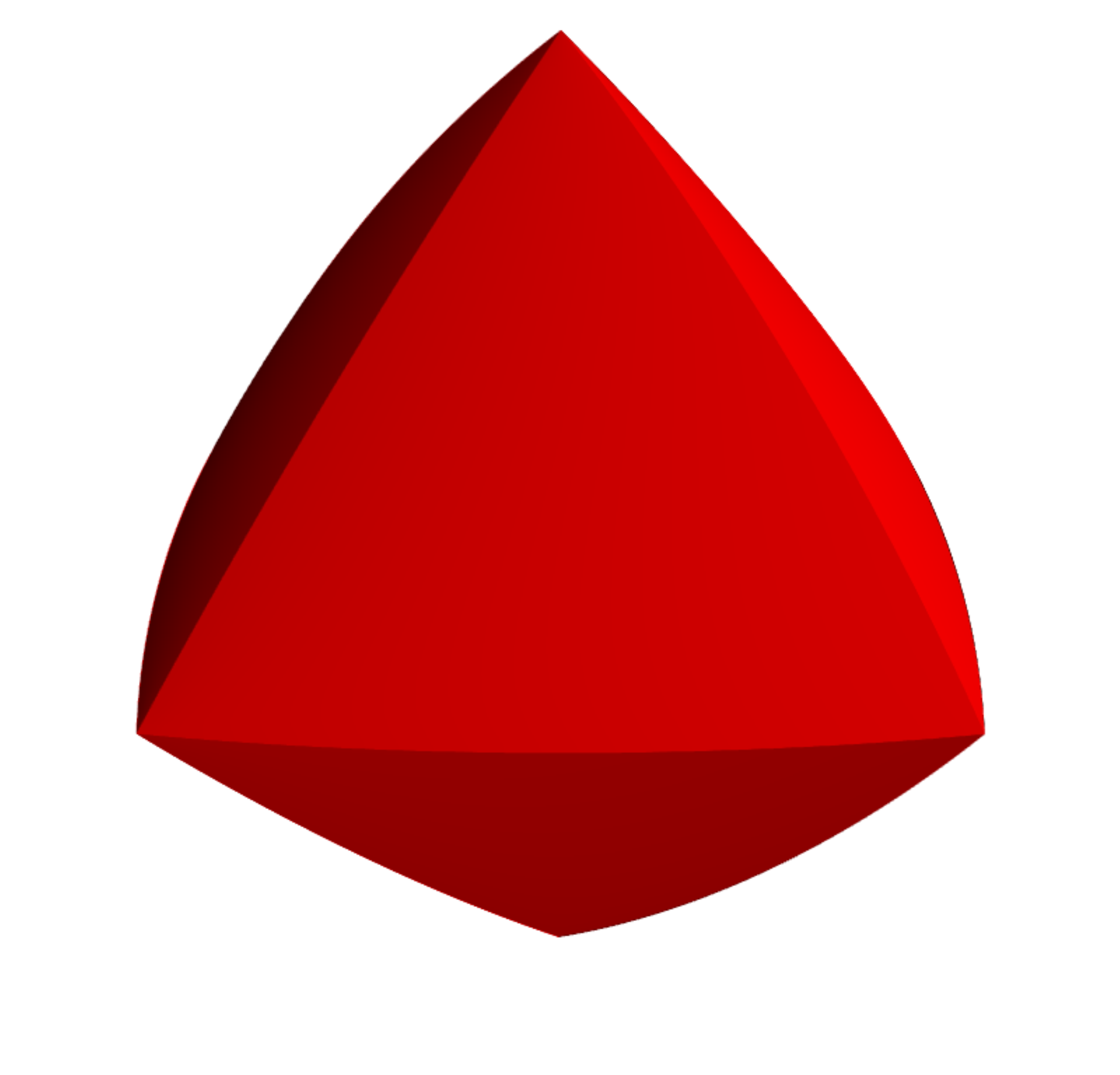}
\includegraphics[width=0.32 \textwidth]{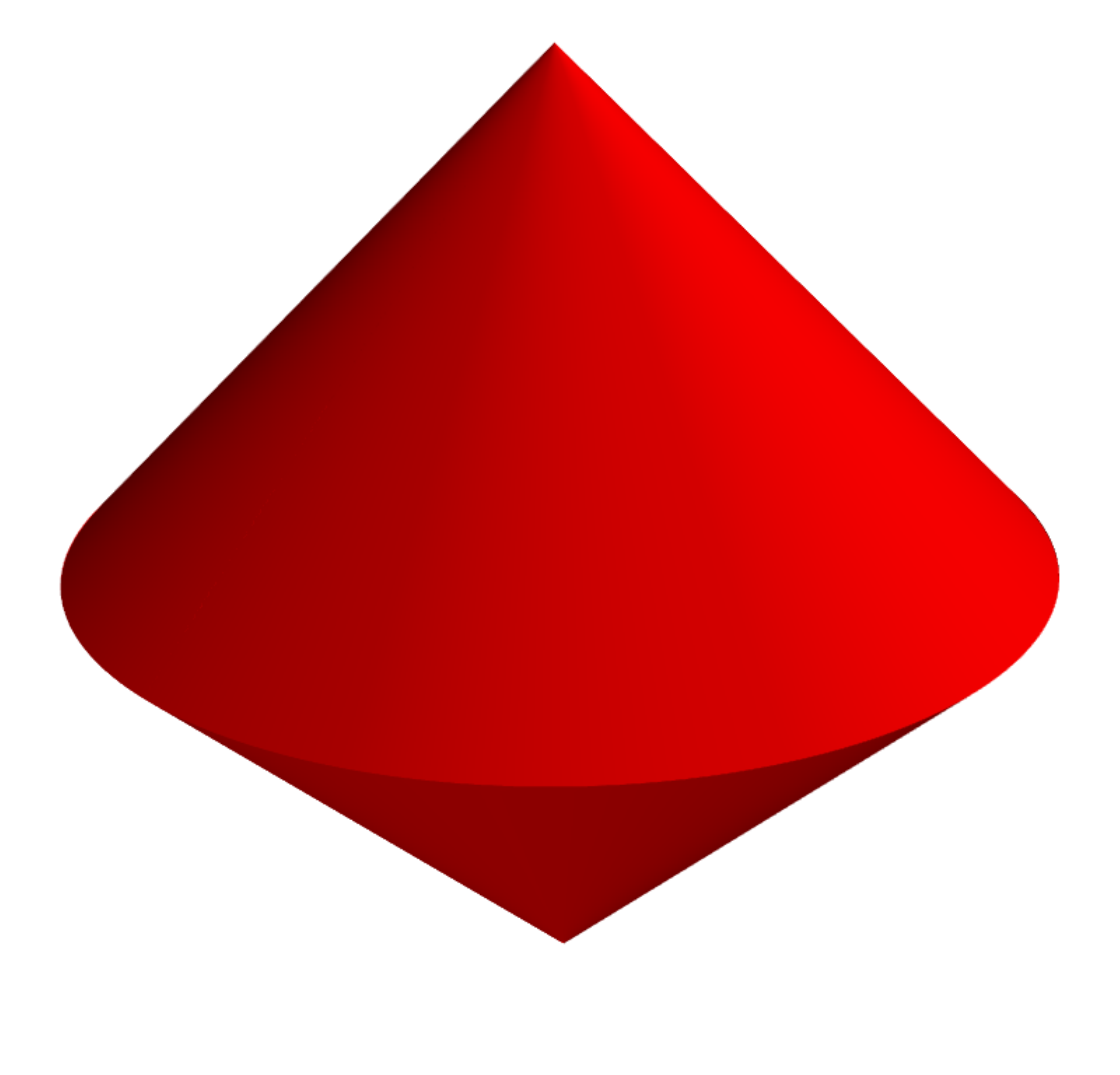}
\caption{Unit balls for various value of $\mathbf{P}^{\top}\mathbf{P}$. See the text for the value of $\mathbf{P}^{\top} \mathbf{P}$. (Best seen in color).}
\label{unit_balls}
\end{figure}

We plot the unit ball of our norm for the following value of $\mathbf{P}^{\top} \mathbf{P}$ (see figure (\ref{unit_balls})):
\begin{equation*}
\left(
\begin{array}{ccc}
1 & 0.9 & 0.1 \\
0.9 & 1 & 0.1 \\
0.1 & 0.1 & 1 
\end{array}
\right)
\hspace{10mm}
\left(
\begin{array}{ccc}
1 & 0.7 & 0.49 \\
0.7 & 1 & 0.7 \\
0.49 & 0.7 & 1 
\end{array}
\right)
\hspace{10mm}
\left(
\begin{array}{ccc}
1 & 1 & 0 \\
1 & 1 & 0 \\
0 & 0 & 1 
\end{array}
\right)
\end{equation*}

We can lower bound and upper bound our norms by the $\ell_2$-norm and $\ell_1$-norm respectively. This shows that, as for the elastic net, our norms interpolate between the $\ell_1$-norm and the $\ell_2$-norm. But the main difference between the elastic net and our norms is the fact that our norms are \emph{adaptive}, and require a single regularization parameter to tune. In particular for the trace Lasso, when two covariates are strongly correlated, it will be close to the $\ell_2$-norm, while when two covariates are almost uncorrelated, it will behave like the $\ell_1$-norm. This is a behavior close to the one of the pairwise elastic net~\cite{lorbert2010}.
\begin{proposition}
\label{bound21}
Let $\mathbf{P} \in \mathbb{R}^{k \times p}$, all of its columns having unit norm. We have
\begin{equation*}
\| \mathbf{w} \|_2 \leq \Omega_{\mathbf{P}} (\mathbf{w}) \leq \| \mathbf{w} \|_1.
\end{equation*}
\end{proposition}

\subsection{Dual norm}
The dual norm is an important quantity for both optimization and theoretical analysis of the estimator. Unfortunately, we are not able in general to obtain a closed form expression of the dual norm for the family of norms we just introduced. However we can obtain a bound, which is exact for some special cases:

\begin{proposition}
The dual norm, defined by
$
\displaystyle{\Omega_{\mathbf{P}}^* (\mathbf{u}) = \max_{\Omega_{\mathbf{P}}(\mathbf{v}) \leq 1} \mathbf{u}^{\top} \mathbf{v}}
$,
can be bounded by:
\begin{equation*}
\Omega^*_{\mathbf{P}} (\mathbf{u}) \leq \| \mathbf{P} \Diag (\mathbf{u}) \|_{op}.
\end{equation*}
\end{proposition}
\begin{proof}
Using the fact that $\diag( \mathbf{P}^{\top} \mathbf{P} ) = \mathbf{1}$, we have 
\begin{align*}
\mathbf{u}^{\top} \mathbf{v} & = \tr \left( \Diag(\mathbf{u}) \mathbf{P}^{\top} \mathbf{P} \Diag(\mathbf{v}) \right) \\
& \leq \| \mathbf{P} \Diag(\mathbf{u}) \|_{op} \| \mathbf{P} \Diag(\mathbf{v}) \|_*,
\end{align*}
where the inequality comes from the fact that the operator norm $\| \cdot \|_{op}$ is the dual norm of the trace norm. The definition of the dual norm then gives the result.
\end{proof}
As a corollary, we can bound the dual norm by a constant times the $\ell_{\infty}$-norm:
\begin{equation*}
\Omega_{\mathbf{P}}^* (\mathbf{u}) \leq \| \mathbf{P} \Diag(\mathbf{u}) \|_{op}
\leq \| \mathbf{P} \|_{op} \| \Diag(\mathbf{u}) \|_{op}
= \| \mathbf{P} \|_{op} \| \mathbf{u} \|_{\infty}.
\end{equation*}
Using proposition (\ref{bound21}), we also have the inequality $\Omega_{\mathbf{P}}^* (\mathbf{u}) \geq \| \mathbf{u} \|_{\infty}$.

\section{Optimization algorithm}
\label{opt}
In this section, we introduce an algorithm to estimate the parameter vector $\mathbf{w}$ when the loss function is equal to the square loss: $\ell(y, \mathbf{w}^{\top}\mathbf{x}) = \frac{1}{2}(y - \mathbf{w}^{\top} \mathbf{x})^2$ and the penalty is the trace Lasso. It is straightforward to extend this algorithm to the family of norms indexed by $\mathbf{P}$. The problem we consider is
\begin{equation*}
\min_{\mathbf{w}} \frac{1}{2} \| \mathbf{y} - \mathbf{X} \mathbf{w} \|_2^2 + \lambda \| \mathbf{X} \Diag(\mathbf{w}) \|_*.
\end{equation*}
We could optimize this cost function by subgradient descent, but this is quite inefficient: computing the subgradient of the trace Lasso is expensive and the rate of convergence of subgradient descent is quite slow. Instead, we consider an iteratively reweighted least-squares method. First, we need to introduce a well-known variational formulation for the trace norm~\cite{argyriou2007}:
\begin{proposition}
Let $\mathbf{M} \in \mathbb{R}^{n \times p}$. The trace norm of $\mathbf{M}$ is equal to:
\begin{equation*}
\| \mathbf{M} \|_{*} = \frac{1}{2} \inf_{\mathbf{S} \succeq 0} \tr \left( \mathbf{M}^{\top} \mathbf{S}^{-1} \mathbf{M} \right) + \tr \left( \mathbf{S} \right),
\end{equation*}
and the infimum is attained for $\mathbf{S} = \left(\mathbf{MM}^{\top} \right)^{1/2}$.
\label{eta-trick}
\end{proposition}
Using this proposition, we can reformulate the previous optimization problem as
\begin{equation*}
\min_{\mathbf{w}} \inf_{S \succeq 0} \frac{1}{2} \| \mathbf{y} - \mathbf{X} \mathbf{w} \|_2^2 + \frac{\lambda}{2}\mathbf{w}^{\top} \Diag \big( \diag (\mathbf{X}^{\top} \mathbf{S}^{-1} \mathbf{X}) \big) \mathbf{w} + \frac{\lambda}{2} \tr (\mathbf{S}).
\end{equation*}
This problem is jointly convex in $(\mathbf{w}, \ \mathbf{S})$~\cite{boyd2004}. In order to optimize this objective function by alternating the minimization over $\mathbf{w}$ and $\mathbf{S}$, we need to add a term $\frac{\lambda \mu_i}{2} \tr(\mathbf{S}^{-1})$. Otherwise, the infimum over $\mathbf{S}$ could be attained at a non invertible $\mathbf{S}$, leading to a non convergent algorithm. The infimum over $\mathbf{S}$ is then attained for $\mathbf{S} = \left( \mathbf{X} \Diag(\mathbf{w})^2 \mathbf{X}^{\top} + \mu_i \mathbf{I} \right)^{1/2}$.

Optimizing over $\mathbf{w}$ is a least-squares problem penalized by a reweighted $\ell_2$-norm equal to $\mathbf{w}^{\top} \mathbf{D} \mathbf{w}$, where $\mathbf{D} = \Diag \left( \diag(\mathbf{X}^{\top} \mathbf{S}^{-1} \mathbf{X}) \right)$. It is equivalent to solving the linear system
\begin{equation*}
(\mathbf{X}^{\top} \mathbf{X} + \lambda \mathbf{D}) \mathbf{w} = \mathbf{X}^{\top} \mathbf{y}.
\end{equation*} 
This can be done efficiently by using a conjugate gradient method. Since the cost of multiplying $(\mathbf{X}^{\top} \mathbf{X} + \lambda \mathbf{D})$ by a vector is $O(np)$, solving the system has a complexity of $O(knp)$, where $k \leq p$ is the number of iterations needed to converge. Using warm restarts, $k$ can be much smaller than $p$, since the linear system we are solving does not change a lot from an iteration to another. Below we summarize the algorithm:

\vspace{3mm}

\noindent
\begin{minipage}{\textwidth}
\begin{center}
\rule{\textwidth}{2pt}
\textsc{Iterative algorithm for estimating $\mathbf{w}$} \\
\rule[1mm]{\textwidth}{1pt}
\end{center}
\vspace{-2mm}
\textbf{Input:} the design matrix $\mathbf{X}$, the initial guess $\mathbf{w}^0$, number of iteration $N$, sequence $\mu_i$.

\textbf{For} $i = 1... N$:

\begin{itemize}
\item Compute the eigenvalue decomposition $\mathbf{U} \Diag(s_k) \mathbf{U}^{\top}$ of $\mathbf{X} \Diag(\mathbf{w}^{i-1})^2 \mathbf{X}^{\top}$.

\item Set $\mathbf{D} = \Diag( \diag (\mathbf{X}^{\top} \mathbf{S}^{-1} \mathbf{X}))$, where $\mathbf{S}^{-1} = \mathbf{U} \Diag(1 / \sqrt{s_k + \mu_i}) \mathbf{U}^{\top}$.

\item Set $\mathbf{w}^i$ by solving the system $(\mathbf{X}^{\top} \mathbf{X} + \lambda \mathbf{D}) \mathbf{w} = \mathbf{X}^{\top} \mathbf{y}$.
\end{itemize}
\rule{\textwidth}{1pt}
\end{minipage}

\vspace{3mm}

For the sequence $\mu_i$, we use a decreasing sequence converging to ten times the machine precision.

\subsection{Choice of $\lambda$}
We now give a method to choose the regularization path. In fact, we know that the vector $\mathbf{0}$ is solution if and only if $\lambda \geq \Omega^* (\mathbf{X}^{\top} \mathbf{y})$~\cite{bach2011}. Thus, we need to start the path at $\lambda = \Omega^* (\mathbf{X}^{\top} \mathbf{y})$, corresponding to the empty solution $\mathbf{0}$, and then decrease $\lambda$. Using the inequalities on the dual norm we obtained in the previous section, we get 
\begin{equation*}
\| \mathbf{X}^{\top} \mathbf{y} \|_{\infty} \leq \Omega^{*}(\mathbf{X}^{\top} \mathbf{y}) \leq \| \mathbf{X} \|_{op} \| \mathbf{X}^{\top} \mathbf{y} \|_{\infty}.
\end{equation*}
Therefore, starting the path at $\lambda = \|\mathbf{X} \|_{op} \| \mathbf{X}^{\top} \mathbf{y} \|_{\infty}$ is a good choice.

\section{Approximation around the Lasso}
\label{dev_lasso}
In this section, we compute the second order approximation of our norm around the special case corresponding to the Lasso. We recall that when $\mathbf{P} = \mathbf{I} \in \mathbb{R}^{p \times p}$, our norm is equal to the $\ell_1$-norm. We add a small perturbation $\Delta \in \mathbb{S}_p$ to the identity matrix, and using Prop.~\ref{prop_perturbation} of the appendix \ref{perturbation_tn}, we obtain the following second order approximation:
\begin{multline*}
\| (\mathbf{I} + \Delta) \Diag(\mathbf{w}) \|_* = \| \mathbf{w} \|_1 + \diag( \Delta)^{\top} | \mathbf{w} | + \\ \sum_{|w_i| > 0} \sum_{|w_j| > 0} \frac{(\Delta_{ji} |w_i| - \Delta_{ij} |w_j|)^2}{4(|w_i| + |w_j|)} + \sum_{|w_i| = 0} \sum_{|w_j| > 0} \frac{(\Delta_{ij} |w_j| )^2}{2 |w_j|} + o(\| \Delta \|^2).
\end{multline*}
We can rewrite this approximation as
\begin{equation*}
\| (\mathbf{I} + \Delta) \Diag(\mathbf{w}) \|_* = \| \mathbf{w} \|_1 + \diag( \Delta)^{\top} | \mathbf{w} | + \sum_{i, j} \frac{\Delta_{ij}^2(|w_i| - |w_j|)^2}{4(|w_i| + |w_j|)} + o(\| \Delta \|^2),
\end{equation*}
using a slight abuse of notation, considering that the last term is equal to $0$ when $w_i = w_j = 0$. The second order term is quite interesting: it shows that when two covariates are correlated, the effect of the trace Lasso is to shrink the corresponding coefficients toward each other. Another interesting remark is the fact that this term is very similar to pairwise elastic net penalties, which are of the form $| \mathbf{w} |^{\top} \mathbf{P} | \mathbf{w} |$, where $\mathbf{P}_{ij}$ is a decreasing function of $\Delta_{ij}$. 

\section{Experiments}
\label{exp}
In this section, we perform synthetic experiments to illustrate the behavior of the trace Lasso and other classical penalties when there are highly correlated covariates in the design matrix. For all experiments, we have $p = 1024$ covariates and $n = 256$ observations. The support $S$ of $\mathbf{w}$ is equal to $\{ 1, ..., k\}$, where $k$ is the size of the support. For $i$ in the support of $\mathbf{w}$, we have $w_i~=~2~(b_i~-~1/2)$, where each $b_i$ is independently drawn from a uniform distribution on $[0, 1]$. The observations $\mathbf{x}_i$ are drawn from a multivariate Gaussian with mean $\mathbf{0}$ and covariance matrix $\Sigma$. For the first experiment, $\Sigma$ is set to the identity, for the second experiment, $\Sigma$ is block diagonal with blocks equal to $0.2 \mathbf{I} + 0.8 \mathbf{11}^{\top}$
corresponding to clusters of eight variables, finally for the third experiment, we set $\Sigma_{ij} = 0.95^{|i - j|}$, corresponding to a Toeplitz design. For each method, we choose the best $\lambda$ for the estimation error, which is reported.

Overall all methods behave similarly in the noiseless and the noisy settings, hence we only report results for the noisy setting. In all three graphs of Figure~2, we observe behaviors that are typical of Lasso, ridge and elastic net: the Lasso performs very well on sparse models, but its performance is rather poor for denser models, almost as poor as the ridge regression. The elastic net offers the best of both worlds since its two parameters allow it to interpolate adaptively between the Lasso and the ridge. In experiment 1, since the variables are uncorrelated, there is no reason to couple their selection. This suggests that the Lasso should be the most appropriate convex regularization. The trace Lasso approaches the Lasso as $n$ goes to infinity, but the weak coupling induced by empirical correlations is sufficient to slightly decrease its performance compared to that of the Lasso. By contrast, in experiments 2 and 3, the trace Lasso outperforms other methods (including the pairwise elastic net) since variables that should be selected together are indeed correlated. As for the penalized elastic net, since it takes into account the correlations between variables it is not surprising that in experiment 2 and 3 it performs better than methods that do not. We do not have a compelling explanation for its superior performance in experiment 1.

\begin{figure}[!h]
\centering
\includegraphics[width=0.32 \textwidth]{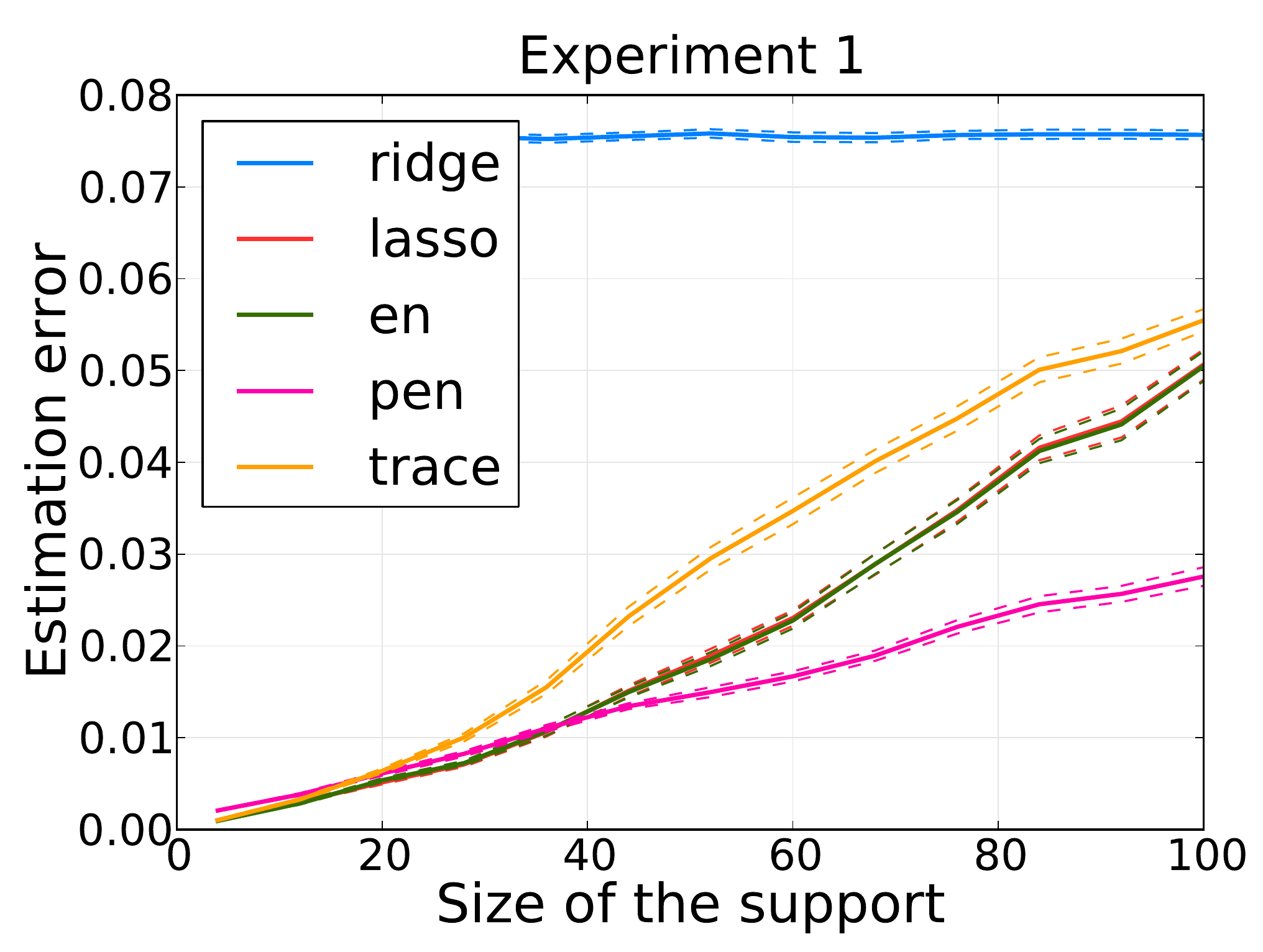} \hfill
\includegraphics[width=0.32 \textwidth]{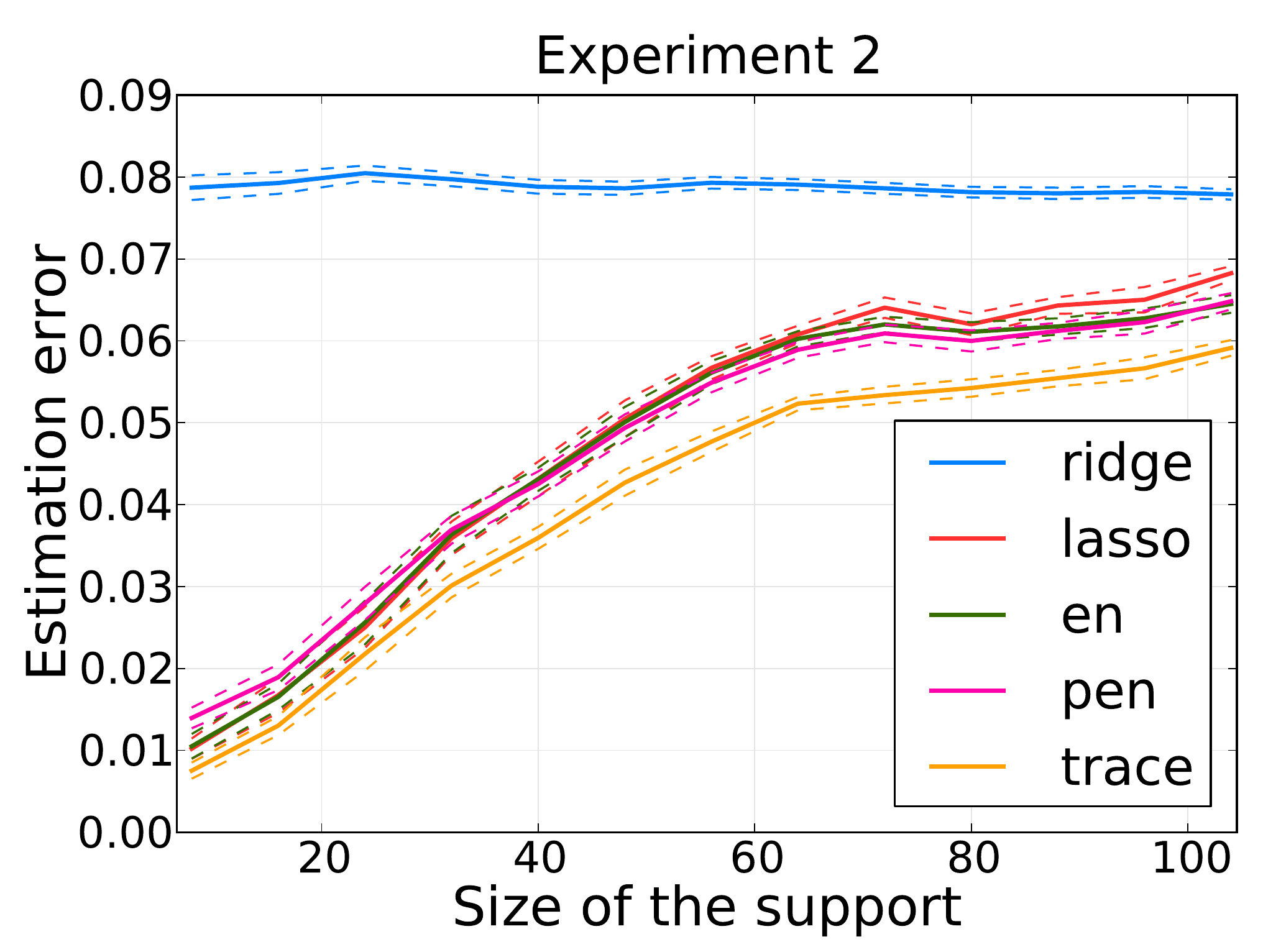} \hfill
\includegraphics[width=0.32 \textwidth]{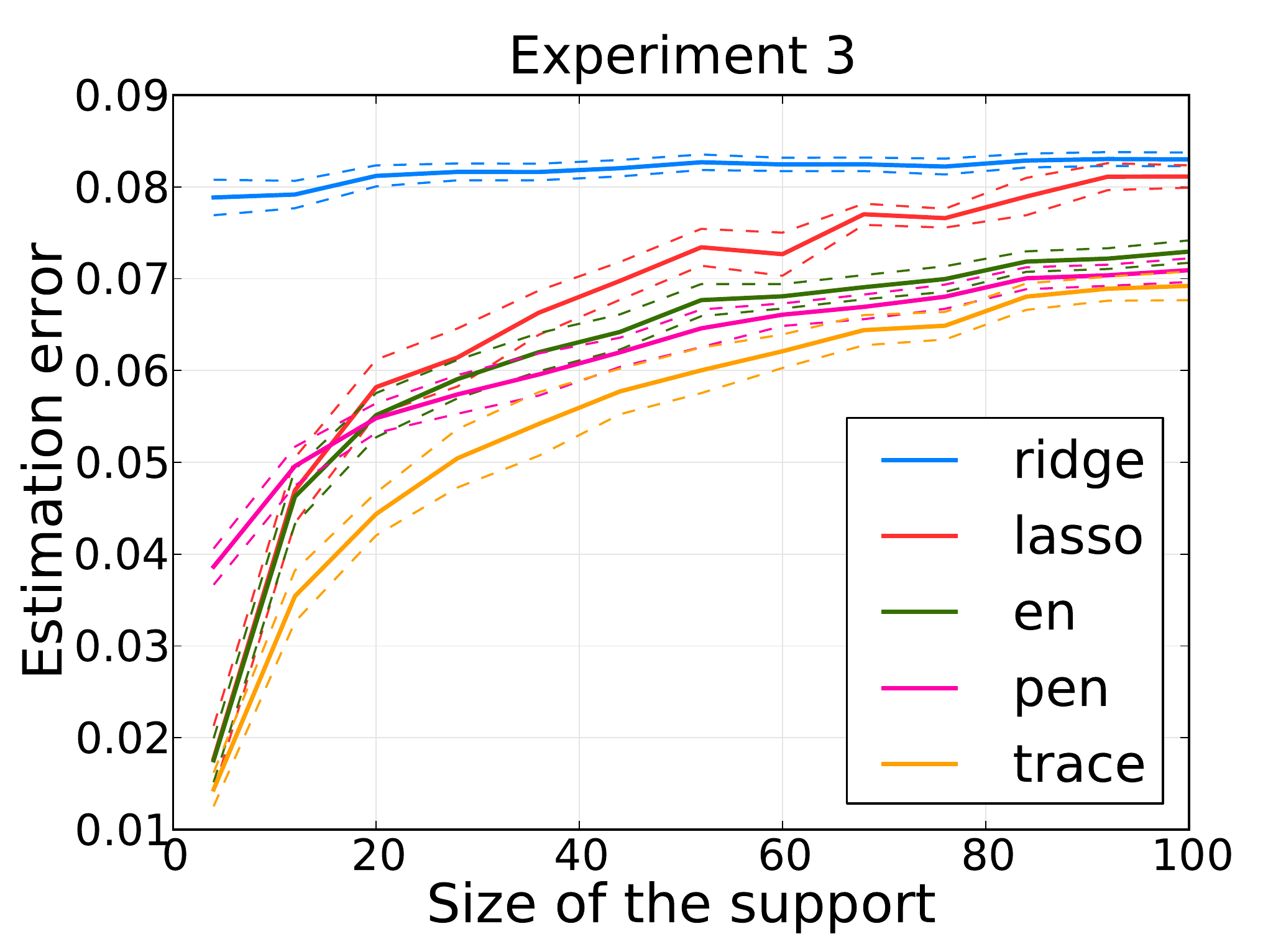}
\caption{Experiment for uncorrelated variables (Best seen in color. \texttt{en} stands for elastic net, \texttt{pen} stands for pairwise elastic net and \texttt{trace} stands for trace Lasso.)}
\end{figure}

\section{Conclusion}
We introduce a new penalty function, the trace Lasso, which takes advantage of the correlation between covariates to add strong convexity exactly in the directions where needed, unlike the elastic net for example, which blindly adds a squared $\ell_2$-norm term in every directions. We show on synthetic data that this adaptive behavior leads to better estimation performance. In the future, we want to show that if a dedicated norm using prior knowledge such as the group Lasso can be used, the trace Lasso will behave similarly and its performance will not degrade too much, providing theoretical guarantees to such adaptivity. Finally, we will seek applications of this estimator in inverse problems such as deblurring, where the design matrix exhibits strong correlation structure.


\section*{Acknowledgements}
This paper was partially supported by the European Research Council (SIERRA Project ERC-239993).

\bibliography{nipstl}{}
\bibliographystyle{unsrt}

\appendix
\newpage
\section{Perturbation of the trace norm}
\label{perturbation_tn}

We follow the technique used in~\cite{bach2008consistency} to obtain an approximation of the trace norm.

\subsection{Jordan-Wielandt matrices}
Let $\mathbf{M} \in \mathbb{R}^{n \times p}$ of rank $r$. We note $s_1 \geq s_2 \geq ... \geq s_r > 0$, the strictly positive singular values of $\mathbf{M}$ and $\mathbf{u}_i$, $\mathbf{v}_i$ the associated left and right singular vectors. We introduce the Jordan-Wielandt matrix
\begin{equation*}
\mathbf{\tilde M} = \left(
\begin{array}{cc}
\mathbf{0} & \mathbf{M} \\
\mathbf{M}^{\top} & \mathbf{0} 
\end{array}
\right) \in \mathbb{R}^{(n+p) \times (n+p)}.
\end{equation*}
The singular values of $\mathbf{M}$ and the eigenvalues of $\mathbf{\tilde M}$ are related:  $\mathbf{\tilde M}$ has eigenvalues $s_i$ and $s_{-i} = -s_i$ associated to eigenvectors 
\begin{equation*}
\mathbf{w}_i = \frac{1}{\sqrt{2}}
\left(
\begin{array}{c}
\mathbf{u}_i \\
\mathbf{v}_i
\end{array}
\right)
\hspace{5mm}
\text{ and }
\hspace{5mm}
\mathbf{w}_{-i} = \frac{1}{\sqrt{2}}
\left(
\begin{array}{c}
\mathbf{u}_i \\
- \mathbf{v}_i
\end{array}
\right).
\end{equation*}
The remaining eigenvalues of $\mathbf{\tilde M}$ are equal to $0$ and are associated to eigenvectors of the form
\begin{equation*}
\mathbf{w} = \frac{1}{\sqrt{2}}
\left(
\begin{array}{c}
\mathbf{u} \\
\mathbf{v}
\end{array}
\right)
\hspace{5mm} \text{and} \hspace{5mm}
\mathbf{w} = \frac{1}{\sqrt{2}}
\left(
\begin{array}{c}
\mathbf{u} \\
\mathbf{-v}
\end{array}
\right),
\end{equation*}
where $\forall \ i \in \{1, ..., r\}, \ \mathbf{u}^{\top}\mathbf{u}_i = \mathbf{v}^{\top}\mathbf{v}_i = 0$.

\subsection{Cauchy residue formula}
Let $\mathcal{C}$ be a closed curve that does not go through the eigenvalues of $\mathbf{\tilde M}$. We define 
\begin{equation*}
\Pi_{\mathcal{C}}(\mathbf{\tilde M}) = \frac{1}{2i \pi} \int_{\mathcal{C}} \lambda (\lambda \mathbf{I} - \mathbf{\tilde M})^{-1} d\lambda.
\end{equation*}
We have
\begin{align*}
\Pi_{\mathcal{C}}(\mathbf{\tilde M})  & = \frac{1}{2 i \pi} \oint \sum_{j} \frac{\lambda}{\lambda - s_j} \mathbf{w}_j \mathbf{w}_j^{\top} d\lambda \\
& = \frac{1}{2 i \pi} \oint \sum_j \left( 1 + \frac{s_j}{\lambda - s_j} \right) \mathbf{w}_j \mathbf{w}_j^{\top} d\lambda \\
& = \sum_{s_j \in \mathcal{C}} s_j \mathbf{w}_j \mathbf{w}_j^{\top}.
\end{align*}

\subsection{Perturbation analysis}
Let $\mathbf{\Delta} \in \mathbb{R}^{n \times p}$ be a perturbation matrix such that $\| \mathbf{\Delta} \|_{op} < s_r / 4$, and let $\mathcal{C}$ be a closed curve around the $r$ largest eigenvalues of $\mathbf{\tilde M}$ and $\mathbf{\tilde M} + \mathbf{\tilde \Delta}$. We can study the perturbation of the strictly positive singular values of $\mathbf{M}$ by computing the trace of $\Pi_{\mathcal{C}}(\mathbf{\tilde M} + \mathbf{\tilde \Delta}) - \Pi_{\mathcal{C}}(\mathbf{\tilde M})$. Using the fact that $(\lambda \mathbf{I} - \mathbf{\tilde M - \tilde \Delta})^{-1} = (\lambda \mathbf{I - \tilde M})^{-1} + (\lambda \mathbf{I - \tilde M})^{-1} \mathbf{\tilde \Delta} (\lambda \mathbf{I} - \mathbf{\tilde M - \tilde \Delta})^{-1}$, we have
\begin{multline*}
\Pi_{\mathcal{C}}(\mathbf{\tilde M} + \mathbf{\tilde \Delta}) - \Pi_{\mathcal{C}}(\mathbf{\tilde M}) = \frac{1}{2 i \pi} \oint \lambda(\lambda \mathbf{I - \tilde M})^{-1} \mathbf{\tilde \Delta} (\lambda \mathbf{I - \tilde M})^{-1}d\lambda \\
+ \frac{1}{2 i \pi} \oint \lambda(\lambda \mathbf{I - \tilde M})^{-1} \mathbf{\tilde \Delta} (\lambda \mathbf{I - \tilde M})^{-1} \mathbf{\tilde \Delta} (\lambda \mathbf{I - \tilde M})^{-1}d\lambda \\
+ \frac{1}{2 i \pi} \oint \lambda(\lambda \mathbf{I - \tilde M})^{-1} \mathbf{\tilde \Delta} (\lambda \mathbf{I - \tilde M})^{-1} \mathbf{\tilde \Delta} (\lambda \mathbf{I - \tilde M - \tilde \Delta})^{-1}d\lambda.
\end{multline*}
We note $A$ and $B$ the first two terms of the right hand side of this equation. We have 
\begin{align*}
\tr (A) & = \sum_{j, k} \tr(\mathbf{w}_j \mathbf{w}_j^{\top} \mathbf{\tilde \Delta} \mathbf{w}_k \mathbf{w}_k^{\top}) \frac{1}{2 i \pi} \oint_{\mathcal{C}} \frac{\lambda d\lambda}{(\lambda - s_j)(\lambda - s_k)} \\
& = \sum_{j} \tr (\mathbf{w}_j^{\top} \mathbf{\tilde \Delta} \mathbf{w}_j) \frac{1}{2 i \pi} \oint_{\mathcal{C}} \frac{\lambda d\lambda}{(\lambda - s_j)^2} \\
& = \sum_{j} \tr(\mathbf{w}_j^{\top} \mathbf{\tilde \Delta} \mathbf{w}_j) \\
& = \sum_{j} \tr(\mathbf{u}_j^{\top} \mathbf{\Delta} \mathbf{v}_j),
\end{align*}
and
\begin{align*}
\tr(B) & = \sum_{j, k, l} \tr(\mathbf{w}_j \mathbf{w}_j^{\top} \mathbf{\tilde \Delta} \mathbf{w}_k \mathbf{w}_k^{\top} \mathbf{\tilde \Delta} \mathbf{w}_l \mathbf{w}_l^{\top}) \frac{1}{2 i \pi} \oint_{\mathcal{C}} \frac{\lambda d\lambda}{(\lambda - s_j)(\lambda - s_k)(\lambda - s_l)} \\
& = \sum_{j, k} \tr(\mathbf{w}_j \mathbf{\tilde \Delta} \mathbf{w}_k \mathbf{w}_k \mathbf{\tilde \Delta} \mathbf{w}_j ) \frac{1}{2 i \pi} \oint_{\mathcal{C}} \frac{\lambda d\lambda}{(\lambda - s_j)^2(\lambda - s_k)}.
\end{align*}
If $s_j = s_k$, the integral is nul. Otherwise, we have
\begin{equation*}
\frac{\lambda}{(\lambda - s_j)^2(\lambda - s_k)} = \frac{a}{\lambda - s_j} + \frac{b}{\lambda - s_k} + \frac{c}{(\lambda - s_j)^2},
\end{equation*}
where
\begin{align*}
a & = \frac{-s_k}{(s_k - s_j)^2}, \\
b & = \frac{s_k}{(s_k - s_j)^2}, \\
c & = \frac{s_j}{s_j - s_k}. 
\end{align*}
Therefore, if $s_j$ and $s_k$ are both inside or outside the interior of $\mathcal{C}$, the integral is equal to zero. So
\begin{align*}
\tr(B) & = \sum_{s_j > 0} \sum_{s_k \leq 0} \frac{-s_k(\mathbf{w}_j^{\top} \mathbf{\tilde \Delta} \mathbf{w}_k)^2}{(s_j - s_k)^2} + \sum_{s_j \leq 0} \sum_{s_k > 0} \frac{s_k(\mathbf{w}_j^{\top} \mathbf{\tilde \Delta} \mathbf{w}_k)^2}{(s_j - s_k)^2} \\
& = \sum_{s_j > 0} \sum_{s_k > 0} \frac{s_k(\mathbf{w}_j^{\top} \mathbf{\tilde \Delta} \mathbf{w}_{-k})^2}{(s_j + s_k)^2} + \sum_{s_j > 0} \sum_{s_k > 0} \frac{s_k(\mathbf{w}_{-j}^{\top} \mathbf{\tilde \Delta} \mathbf{w}_k)^2}{(s_j + s_k)^2} + \sum_{s_j = 0} \sum_{s_k > 0} \frac{(\mathbf{w}_j^{\top} \mathbf{\tilde \Delta} \mathbf{w}_k)^2}{s_k} \\
& = \sum_{s_j > 0} \sum_{s_k > 0} \frac{(\mathbf{w}_{-j}^{\top} \mathbf{\tilde \Delta} \mathbf{w}_k)^2}{s_j + s_k} + \sum_{s_j = 0} \sum_{s_k > 0} \frac{(\mathbf{w}_j^{\top} \mathbf{\tilde \Delta} \mathbf{w}_k)^2}{s_k}.
\end{align*}
For $s_j > 0$ and $s_k > 0$, we have
\begin{equation*}
\mathbf{w}_{-j}^{\top} \mathbf{\tilde \Delta} \mathbf{w}_k = \frac{1}{2} \left( \mathbf{u}_j^{\top} \mathbf{\Delta} \mathbf{v}_k - \mathbf{u}_k^{\top} \mathbf{\Delta} \mathbf{v}_j \right),
\end{equation*}
and for $s_j = 0$ and $s_k > 0$, we have
\begin{equation*}
\mathbf{w}_j^{\top} \mathbf{\tilde \Delta} \mathbf{w}_k = \frac{1}{2} \left( \pm \mathbf{u}_k^{\top} \mathbf{\Delta} \mathbf{v}_j + \mathbf{u}_j^{\top} \mathbf{\Delta} \mathbf{v}_k \right).
\end{equation*}
So 
\begin{equation*}
\tr(B) = \sum_{s_j > 0} \sum_{s_k > 0} \frac{(\mathbf{u}_j^{\top} \mathbf{\Delta} \mathbf{v}_k - \mathbf{u}_k^{\top} \mathbf{\Delta} \mathbf{v}_j)^2}{4(s_j + s_k)} + \sum_{s_j = 0} \sum_{s_k > 0} \frac{(\mathbf{u}_k^{\top} \mathbf{\Delta} \mathbf{v}_j)^2 + (\mathbf{u}_j^{\top} \mathbf{\Delta} \mathbf{v}_k)^2}{2 s_k}.
\end{equation*}

Now, let $\mathcal{C}_0$ be the circle of center $\mathbf{0}$ and radius $s_r / 2$. We can study the perturbation of the singular values of $\mathbf{M}$ equal to zero by computing the trace norm of $\Pi_{\mathcal{C}_0}(\mathbf{\tilde M} + \mathbf{\tilde \Delta}) - \Pi_{\mathcal{C}_0}(\mathbf{\tilde M})$. We have
\begin{multline*}
\Pi_{\mathcal{C}_0}(\mathbf{\tilde M} + \mathbf{\tilde \Delta}) - \Pi_{\mathcal{C}_0}(\mathbf{\tilde M}) = \frac{1}{2 i \pi} \oint_{\mathcal{C}_0} \lambda(\lambda \mathbf{I - \tilde M})^{-1} \mathbf{\tilde \Delta} (\lambda \mathbf{I - \tilde M})^{-1}d\lambda \\
+ \frac{1}{2 i \pi} \oint_{\mathcal{C}_0} \lambda(\lambda \mathbf{I - \tilde M})^{-1} \mathbf{\tilde \Delta} (\lambda \mathbf{I - \tilde M})^{-1} \mathbf{\tilde \Delta} (\lambda \mathbf{I - \tilde M})^{-1}d\lambda \\
+ \frac{1}{2 i \pi} \oint_{\mathcal{C}_0} \lambda(\lambda \mathbf{I - \tilde M})^{-1} \mathbf{\tilde \Delta} (\lambda \mathbf{I - \tilde M})^{-1} \mathbf{\tilde \Delta} (\lambda \mathbf{I - \tilde M - \tilde \Delta})^{-1}d\lambda.
\end{multline*}
Then, if we note the first integral $C$ and the second one $D$, we get
\begin{align*}
C = \sum_{j,k} \mathbf{w}_j \mathbf{w}_j^{\top} \mathbf{\tilde \Delta} \mathbf{w}_k \mathbf{w}_k^{\top} \frac{1}{2 i \pi} \oint_{\mathcal{C}_0} \frac{\lambda d\lambda}{(\lambda - s_j)(\lambda - s_k)}.
\end{align*}
If both $s_j$ and $s_k$ are outside $int(\mathcal{C}_0)$, then the integral is equal to zero. If one of them is inside, say $s_j$, then $s_j = 0$ and the integral is equal to 
\begin{equation*}
\oint_{\mathcal{C}_0} \frac{d\lambda}{\lambda - s_k}
\end{equation*}
Then this integral is non nul if and only if $s_k$ is also inside $int(\mathcal{C}_0)$. Thus
\begin{align*}
C & = \sum_{j, k} \mathbf{w}_j \mathbf{w}_j^{\top} \mathbf{\tilde \Delta} \mathbf{w}_k \mathbf{w}_k^{\top} \ind_{s_j \in int(\mathcal{C}_0)} \ind_{s_k \in int(\mathcal{C}_0)} \\
& = \sum_{s_j = 0} \sum_{s_k = 0} \mathbf{w}_j \mathbf{w}_j^{\top} \mathbf{\tilde \Delta} \mathbf{w}_k \mathbf{w}_k^{\top} \\
& = \mathbf{W}_0 \mathbf{W}_0^{\top} \mathbf{\tilde \Delta} \mathbf{W}_0 \mathbf{W}_0^{\top},
\end{align*}
where $\mathbf{W}_0$ are the eigenvectors associated to the eigenvalue $0$. We have
\begin{equation*}
D = \sum_{j, k, l} \mathbf{w}_j \mathbf{w}_j^{\top} \mathbf{\tilde \Delta} \mathbf{w}_k \mathbf{w}_k^{\top} \mathbf{\tilde \Delta} \mathbf{w}_l \mathbf{w}_l^{\top} \frac{1}{2 i \pi} \oint_{\mathcal{C}_0} \frac{\lambda d\lambda}{(\lambda - s_j)(\lambda - s_k)(\lambda - s_l)}.
\end{equation*}
The integral is not equal to zero if and only if exactly one eigenvalue, say $s_i$, is outside $int(\mathcal{C}_0)$. The integral is then equal to $-1 / s_i$. Thus
\begin{multline*}
D = - \mathbf{W}_0 \mathbf{W}_0^{\top} \mathbf{\tilde \Delta} \mathbf{W}_0 \mathbf{W}_0^{\top} \mathbf{\tilde \Delta} \mathbf{W} S^{-1} \mathbf{W}^{\top} 
- \mathbf{W} S^{-1} \mathbf{W}^{\top} \mathbf{\tilde \Delta} \mathbf{W}_0 \mathbf{W}_0^{\top} \mathbf{\tilde \Delta} \mathbf{W}_0 \mathbf{W}_0^{\top} \\
- \mathbf{W}_0 \mathbf{W}_0^{\top} \mathbf{\tilde \Delta} \mathbf{W} S^{-1} \mathbf{W}^{\top} \mathbf{\tilde \Delta} \mathbf{W}_0 \mathbf{W}_0^{\top},
\end{multline*}
where $S = \Diag(-\mathbf{s}, \ \mathbf{s})$. Finally, putting everything together, we get

\vspace{3mm}

\begin{proposition}
\label{prop_perturbation}
Let $\mathbf{M} = \mathbf{U} \Diag(\mathbf{s}) \mathbf{V}^{\top} \in \mathbb{R}^{n \times p}$, the singular value decomposition of $\mathbf{M}$, with $\mathbf{U} \in \mathbb{R}^{n \times r}$, $\mathbf{V} \in \mathbb{R}^{p \times r}$. Let $\mathbf{\Delta} \in \mathbb{R}^{n \times p}$. We have
\begin{multline*}
\| \mathbf{M + \Delta} \|_* = \| \mathbf{M} \|_* + \| \mathbf{Q} \|_* + \tr (\mathbf{V} \mathbf{U}^{\top} \mathbf{\Delta}) + \\
\sum_{s_j > 0} \sum_{s_k > 0} \frac{(\mathbf{u}_j^{\top} \mathbf{\Delta} \mathbf{v}_k - \mathbf{u}_k^{\top} \mathbf{\Delta} \mathbf{v}_j)^2}{4(s_j + s_k)} + \sum_{s_j = 0} \sum_{s_k > 0} \frac{(\mathbf{u}_{k}^{\top} \mathbf{\Delta} \mathbf{v}_{0j})^2 + (\mathbf{u}_{0j}^{\top} \mathbf{\Delta} \mathbf{v}_{k})^2}{2 s_k} + o(\| \mathbf{\Delta} \|^2),
\end{multline*}
where
\begin{multline*}
\mathbf{Q} = \mathbf{U}_0^{\top} \Delta \mathbf{V}_0 \ - \ \mathbf{U}_0^{\top} \Delta \mathbf{V}_0 \mathbf{V}_0^{\top} \Delta^{\top} \mathbf{U} \Diag(\mathbf{s})^{-1}\\
- \ \Diag(\mathbf{s})^{-1} \mathbf{V}^{\top} \Delta^{\top} \mathbf{U}_0 \mathbf{U}_0^{\top} \Delta \mathbf{V}_0 \ - \ \mathbf{U}_0^{\top} \Delta \mathbf{V} \Diag(\mathbf{s})^{-1} \mathbf{U}^{\top} \Delta \mathbf{V}_0.
\end{multline*}
\label{tldl}
\end{proposition}

\section{Proof of proposition 1}
\label{proof_prop_1}
In this section, we prove that if the loss function is strongly convex with respect to its second argument, then the solution of the penalized empirical risk minimization is unique.

\vspace{3mm}

Let $\mathbf{\hat w} \in \argmin_{\mathbf{w}} \sum_{i=1}^n \ell(y_i, \mathbf{w}^{\top} \mathbf{x}_i) + \lambda \| \mathbf{X} \Diag (\mathbf{w}) \|_*.$ If $\mathbf{\hat w}$ is in the nullspace of $\mathbf{X}$, then $\mathbf{\hat w} = 0$ and the minimum is unique. From now on, we suppose that the minima are not in the nullspace of $\mathbf{X}$.

\vspace{3mm}

Let $\mathbf{u}, \mathbf{v} \in \argmin_{\mathbf{w}} \sum_{i=1}^n \ell(y_i, \mathbf{w}^{\top} \mathbf{x}_i) + \lambda \| \mathbf{X} \Diag (\mathbf{w}) \|_*$ and $\delta = \mathbf{v} - \mathbf{u}$. By convexity of the objective function, all the $\mathbf{w} = \mathbf{u} + t \delta$, for $t \in ]0, 1[$ are also optimal solutions, and so, we can choose an optimal solution $\mathbf{w}$ such that $ w_i \neq 0$ for all $i$ in the support of $\delta$. Because the loss function is strongly convex outside the nullspace of $\mathbf{X}$, $\delta$ is in the nullspace of $\mathbf{X}$.

\vspace{3mm}

Let $\mathbf{X} \Diag(\mathbf{w}) = \mathbf{U} \Diag (\mathbf{s}) \mathbf{V}^{\top}$ be the SVD of $\mathbf{X} \Diag(\mathbf{w})$. We have the following development around $\mathbf{w}$:
\begin{multline*}
\| \mathbf{X} \Diag (\mathbf{w} + t \delta) \|_* = \| \mathbf{X} \Diag (\mathbf{w})\|_*
 + \tr (\Diag(t \delta) \mathbf{X}^{\top} \mathbf{U} \mathbf{V}^{\top}) + \\
\sum_{s_i > 0} \sum_{s_j > 0} \frac{\tr (\Diag(t \delta) \mathbf{X}^{\top} (\mathbf{u}_i \mathbf{v}_j^{\top} - \mathbf{u}_j \mathbf{v}_i^{\top}))^2}{4(s_i + s_j)} + \sum_{s_i > 0} \sum_{s_j = 0}\frac{\tr (\Diag(t \delta) \mathbf{X}^{\top} \mathbf{u}_i \mathbf{v}_j^{\top})^2}{2 s_i} + o(t^2).
\end{multline*}
We note $S$ the support of $\mathbf{w}$. Using the fact that the support of $\delta$ is included in $S$, we have $\mathbf{X} \Diag(t \delta) = \mathbf{X} \Diag(\mathbf{w}) \Diag (t \gamma)$, where $\gamma_i = \frac{\delta_i}{w_i}$ for $i \in S$ and $0$ otherwise. Then:
\begin{multline*}
\| \mathbf{X} \Diag (\mathbf{w} + t \delta) \|_* = \| \mathbf{X} \Diag (\mathbf{w})\|_*
 + t \gamma^{\top} \diag(\mathbf{V} \Diag(\mathbf{s}) \mathbf{V}^{\top}) + \\
\sum_{s_i > 0} \sum_{s_j > 0} \frac{t^2 \tr \left((s_i - s_j) \Diag(\gamma) \mathbf{v}_i \mathbf{v}_j^{\top}\right)^2}{4(s_i + s_j)} + \sum_{s_i > 0} \sum_{s_j = 0}\frac{t^2 \tr \left( s_i \Diag( \gamma ) \mathbf{v}_i \mathbf{v}_j^{\top} \right)^2}{2 s_i} + o(t^2).
\end{multline*}
For small $t$, $\mathbf{w} + t \delta$ is also a minimum, and therefore, we have:
\begin{align}
& \forall \ s_i > 0, \ s_j > 0, \hspace{5mm} (s_i - s_j) \tr \left( \Diag (\gamma) \mathbf{v}_i \mathbf{v}_j^{\top} \right) = 0, \\
& \forall \ s_i > 0, \ s_j = 0, \hspace{5mm} \tr \left( \Diag(\gamma) \mathbf{v}_i \mathbf{v}_j^{\top} \right) = 0.
\end{align}
This could be summarized as
\begin{equation}
\forall \ s_i \neq s_j, \hspace{5mm} \mathbf{v}_i^{\top} (\Diag(\gamma) \mathbf{v}_j) = 0.
\end{equation}
 This means that the eigenspaces of $\Diag(\mathbf{w}) \mathbf{X}^{\top} \mathbf{X} \Diag(\mathbf{w})$ are stable by the matrix $\Diag(\gamma)$. Therefore, $\Diag(\mathbf{w}) \mathbf{X}^{\top} \mathbf{X} \Diag(\mathbf{w})$ and $\Diag(\gamma)$ are simultaneously diagonalizable and so, they commute. Therefore:
\begin{equation}
\label{gamma_const}
\forall \ i, j \in S, \ \ \ \sigma_{ij} \gamma_i = \sigma_{ij} \gamma_j
\end{equation}
where $\sigma_{ij} = [\mathbf{X}^{\top} \mathbf{X}]_{ij}$. We define a partition $(S_k)$ of $S$, such that $i$ and $j$ are in the same set $S_k$ if there exists a path $i = a_1, ..., a_m = j$ such that $\sigma_{a_n, a_{n+1}} \neq 0$ for all $n \in \{1, ..., m-1 \}$. Then, using equation (\ref{gamma_const}), $\gamma$ is constant on each $S_k$. $\delta$ being in the nullspace of $\mathbf{X}$, we have:
\begin{align}
0 & = \delta^{\top} \mathbf{X}^{\top} \mathbf{X} \delta \\
& = \sum_{S_k} \sum_{S_l} \delta_{S_k}^{\top} \mathbf{X}^{\top} \mathbf{X} \delta_{S_l} \\
& = \sum_{S_k} \delta_{S_k}^{\top} \mathbf{X}^{\top} \mathbf{X} \delta_{S_k} \\
& = \sum_{S_k} \| \mathbf{X} \delta_{S_k} \|_2^2.
\end{align}
So for all $S_i$, $\mathbf{X} \delta_{S_i} = 0$. Since a predictor $\mathbf{X}_i$ is orthogonal to all the predictors belonging to other groups defined by the partition $(S_k)$, we can decompose the norm $\Omega$:
\begin{equation}
\| \mathbf{X} \Diag(\mathbf{w}) \|_* = \sum_{S_k} \| \mathbf{X} \Diag(\mathbf{w}_{S_k}) \|_*.
\end{equation}
We recall that $\gamma$ is constant on each $S_k$ and so $\delta_{S_k}$ is colinear to $\mathbf{w}_{S_i}$, by definition of $\gamma$. If $\delta_{S_i}$ is not equal to zero, this means that $\mathbf{w}_{S_i}$, which is not equal to zero, is in the nullspace of $\mathbf{X}$. Replacing $\mathbf{w}_{S_i}$ by $0$ will not change the value of the data fitting term but it will strictly decreases the value of the norm $\Omega$. This is a contradiction with the optimality of $\mathbf{w}$. Thus all the $\delta_{S_i}$ are equal to zero and the minimum is unique.

\section{Proof of proposition 3}
For the first inequality, we have
\begin{align*}
\| \mathbf{w} \|_2 & = \| \mathbf{P} \Diag (\mathbf{w}) \|_F \\
& \leq \| \mathbf{P} \Diag (\mathbf{w}) \|_*.
\end{align*}
For the second inequality, we have
\begin{align*}
\| \mathbf{P} \Diag(\mathbf{w}) \|_* & = \max_{\| \mathbf{M} \|_{op} \leq 1} \tr \left( \mathbf{M}^{\top} \mathbf{P} \Diag(\mathbf{w}) \right) \\
& = \max_{\| \mathbf{M} \|_{op} \leq 1} \diag \left( \mathbf{M}^{\top} \mathbf{P} \right)^{\top} \mathbf{w} \\
& \leq \max_{\| \mathbf{M} \|_{op} \leq 1} \sum_{i=1}^p | \mathbf{M}^{(i) \top} \mathbf{P}^{(i)} | \ | w_i | \\
& \leq \| \mathbf{w} \|_1.
\end{align*}
The first equality is the fact that the dual norm of the trace norm is the operator norm and the second inequality uses the fact that all matrices of operator norm smaller than one have columns of $\ell_2$ norm smaller than one.

\end{document}